\documentclass[11pt]{article}

\usepackage{hyperref}
\usepackage{algorithm}
\usepackage{algorithmic}
\usepackage{amsmath}
\usepackage{latexsym}
\usepackage{amssymb}
\usepackage{amsthm}
\usepackage{amsmath}
\usepackage{xspace}
\usepackage{mdwlist}
\usepackage{graphicx}
\usepackage{color}
\usepackage{boxedminipage}
\usepackage{natbib}

\title{Near-Optimal Algorithms for Online Matrix Prediction}

  \author{Elad Hazan\thanks{Technion - Israel Institute of Technology. \texttt{ehazan@ie.technion.ac.il}.}
\and
Satyen Kale\thanks{IBM T. J. Watson Research Center. \texttt{\small sckale@us.ibm.com}.}
\and
 Shai Shalev-Shwartz\thanks{Hebrew University. \texttt{\small shais@cs.huji.ac.il}.}
 }

\date{}

\newcommand{\ignore}[1]{}

\newtheorem{Def}{Definition}

\newtheorem{Thm}{Theorem}
\newtheorem{Lem}[Thm]{Lemma}
\newtheorem{Cor}[Thm]{Corollary}

\newcommand{\sym}{\ensuremath{{\mathsf{sym}}}}
\newcommand{\abs}{\ensuremath{{\mathsf{abs}}}}

\def\reals{{\mathbb R}}
\def\sS{{\mathbb S}}
\newcommand\E{\mbox{\bf E}}

\newcommand\av{\mathop{\mathbb E}}

\def\bB{\mathbf{B}}

\def\bv{\mathbf{v}}

\def\bA{\mathbf{A}}
\def\bD{\mathbf{D}}

\def\bI{\mathbf{I}}

\def\bL{\mathbf{L}}

\def\bN{\mathbf{N}}
\def\bP{\mathbf{P}}
\def\bQ{\mathbf{Q}}
\def\bT{\mathbf{T}}
\def\bU{\mathbf{U}}
\def\bV{\mathbf{V}}
\def\bW{\mathbf{W}}
\def\bw{\mathbf{w}}
\def\bX{\mathbf{X}}
\def\bY{\mathbf{Y}}
\def\sgn{\text{sgn}}
\def\eye{\mathbf{I}}
\def\bone{\mathbf{1}}
\def\bzero{\mathbf{0}}
\def\bSigma{\mathbf{\Sigma}}
\def\balpha{{\boldsymbol \alpha}}

\def\mL{{\mathcal L}}
\def\mW{{\mathcal W}}

\def\epsilon{\varepsilon}
\def\tsum{{\textstyle \sum}}

\newcommand{\K}{\ensuremath{\mathcal K}}

\def\Tr{\mathrm{Tr}}

\newcommand{\thmref}[1]{Theorem~\ref{#1}}
\newcommand{\lemref}[1]{Lemma~\ref{#1}}
\renewcommand{\eqref}[1]{Eq.~(\ref{#1})}

\newcommand{\secref}[1]{Section~\ref{#1}}

\newcommand{\conv}{\mathrm{conv}}
\newcommand{\half}{\frac{1}{2}}
\newcommand{\thalf}{\tfrac{1}{2}}

\newenvironment{myalgo}[1]%
{
\begin{center}
\begin{boxedminipage}{0.8\linewidth}
\begin{center}
\textbf{\texttt{#1}}
\end{center}
\rm
\begin{tabbing}
....\=...\=...\=...\=...\=  \+ \kill
} %
{\end{tabbing} 
\end{boxedminipage} \end{center} 
}

\begin{document}
\maketitle

\begin{abstract}
  In several online prediction problems of recent interest the
  comparison class is composed of matrices with bounded entries. For
  example, in the online max-cut problem, the comparison class is
  matrices which represent cuts of a given graph and in online
  gambling the comparison class is matrices which represent
  permutations over $n$ teams. Another important example is online
  collaborative filtering in which a widely used comparison class is
  the set of matrices with a small trace norm. In this paper we
  isolate a property of matrices, which we call
  $(\beta,\tau)$-decomposability, and derive an efficient online
  learning algorithm, that enjoys a regret bound of
  $\tilde{O}(\sqrt{\beta\,\tau\,T})$ for all problems in which the
  comparison class is composed of $(\beta,\tau)$-decomposable
  matrices. By analyzing the decomposability of cut matrices,
  triangular matrices, and low trace-norm matrices, we derive near optimal
  regret bounds for online max-cut, online gambling, and online
  collaborative filtering. In particular, this resolves (in the
  affirmative) an open problem posed by \citet{Jake,KNW}.  Finally, we
  derive lower bounds for the three problems and show that our upper
  bounds are optimal up to logarithmic factors. In particular, our lower bound
  for the online collaborative filtering problem resolves another open
  problem posed by \citet{ShamirSr11}.
\end{abstract}

\section{Introduction}

We consider online learning problems in which on each round the
learner receives $(i_t,j_t) \in [m] \times [n]$ and should return a
prediction in $[-1,1]$.  For example, in the online collaborative
filtering problem, $m$ is the number of users, $n$ is the number of
items (e.g., movies), and on each online round the learner should
predict a number in $[-1,1]$ indicating how much user $i_t \in [m]$ likes
item $j_t \in [n]$. Once the learner makes the prediction, the
environment responds with a loss function, $\ell_t : [-1,1] \to
\mathbb{R}$, that assesses the correctness of the learner's prediction. 

A natural approach for the learner is to maintain a matrix $\bW_t \in
[-1,1]^{m \times n}$, and to predict the corresponding entry,
$W_t(i_t,j_t)$.  The matrix is updated based on the loss function
and the process continues. 

Without further structure, the above setting is equivalent to $mn$ independent prediction problems - one per user-item
pair. However, it is usually assumed that there is a relationship
between the different matrix entries - e.g. similar users prefer
similar movies. This can be modeled in the online learning setting by
assuming that there is some fixed matrix $\bW$, in a restricted class
of matrices $\mW \subseteq [-1,1]^{m\times n}$, such that the strategy
which always predicts $W(i_t,j_t)$ has a small cumulative loss. A
common choice for $\mW$ in the collaborative filtering application is
to be the set of matrices with a trace norm of at most $\tau$ (which
intuitively requires the prediction matrix to be of low rank). As
usual, rather than assuming that some $\bW \in \mW$ has a small
cumulative loss, we require that the regret of the online learner with
respect to $\mW$ will be small. Formally, after $T$ rounds, the regret
of the learner is
\[ \text{Regret}\ :=\ \sum_{t=1}^T \ell_t(W_t(i_t,j_t)) - \min_{\bW \in \mW} \sum_{t=1}^T \ell_t(W(i_t, j_t)),\]
and we would like the regret to be as small as possible.

A natural question is what properties of $\mW$ enables us to derive an
\emph{efficient} online learning algorithm that enjoys low regret, and
how does the regret depend on the properties of $\mW$. In this paper
we define a property of matrices, called
$(\beta, \tau)$-decomposability, and derive an efficient online
learning algorithm that enjoys a regret bound of
$\tilde{O}(\sqrt{\beta\,\tau\,T})$ for any problem in which $\mW
\subset [-1,1]^{m \times n}$ and every matrix $\bW \in \mW$ is
$(\beta, \tau)$-decomposable. Roughly speaking, $\bW$ is
$(\beta, \tau)$-decomposable if a symmetrization of it can be written
as $\bP-\bN$ where both $\bP$ and $\bN$ are positive semidefinite,
have sum of traces bounded by $\tau$, and have diagonal elements bounded by
$\beta$.

We apply this technique to three online learning problems. 
\begin{enumerate}
\item \textbf{Online max-cut}: On each round, the learner receives a pair
of graph nodes $(i,j) \in [n]\times[n]$, and should decide whether
there is an edge connecting $i$ and $j$. Then, it receives a binary
feedback. The comparison class is the set of all cuts of the graph,
which can be encoded as the set of matrices $\{\bW_A : A \subset
[n]\}$, where $W_A(i,j)$ indicates if $(i,j)$ crosses the cut
defined by $A$ or not. It is possible to achieve a regret of
$O(\sqrt{n T})$ for this problem by a non-efficient algorithm (simply
refer to each $A$ as an expert and apply a prediction with expert
advice algorithm).  Our algorithm yields a nearly optimal regret bound of
$O(\sqrt{n \log(n) T})$ for this problem. This is the first
\emph{efficient} algorithm that achieves near optimal regret. 
\item \textbf{Online Gambling}: On each round, the learner receives a
  pair of teams $(i,j) \in [n] \times [n]$, and should predict whether
  $i$ is going to beat $j$ in an upcoming matchup or vice versa. The
  comparison class is the set of permutations over the teams, where a
  permutation will predict that $i$ is going to beat $j$ if $i$
  appears before $j$ in the permutation. Permutations can be encoded
  naturally as matrices, where $W(i,j)$ is either $1$ (if $i$
  appears before $j$ in the permutation) or $0$. Again, it is possible
  to achieve a regret of $O(\sqrt{n \log(n) T})$ by a non-efficient
  algorithm (that simply treats each permutation as an expert). Our
  algorithm yields a nearly optimal regret bound of $O(\sqrt{n
    \log^3(n) T})$. This resolves an open problem posed in
  \citet{Jake,KNW}. Achieving this kind of regret bound was widely considered {\em intractable}, since computing the best permutation in hindsight is exactly the {\bf NP}-hard minimum feedback arc set problem. In fact, \citet{kanade} tried to show computational hardness for this problem by reducing the problem of online agnostic learning of halfspaces in a restricted setting to it. This paper shows that the problem is in fact tractable.

\item \textbf{Online Collaborative Filtering}: We already mentioned
  this problem previously. We consider the comparison class 
  $\mW =
  \{\bW \in [-1,1]^{m \times n} : \|\bW\|_\star \le \tau\}$, where
  $\|\cdot\|_\star$ is the trace norm. Without loss of generality
  assume $m \le n$. Our algorithm yields a nearly optimal regret bound
  of $O(\sqrt{\tau \sqrt{n} \log(n) T})$.  Since for this problem one
  typically has $\tau = \Theta(n)$, we can rewrite
  the regret bound as $O(\sqrt{n^{3/2} \log(n) T})$.  In contrast, a
  direct application of the online mirror descent framework to this
  problem yields a regret of $O(\sqrt{\tau^2 T}) = O(\sqrt{n^2
    T})$. The latter is a trivial bound since the bound becomes
  meaningful only after $T\ge n^2$ rounds (which means that we saw the
  entire matrix). 

  Recently, \citet{CesaBianchiShamir2011} proposed a rather different
  algorithm with regret bounded by $O({\tau \sqrt{n}})$ but under the
  additional assumption that each entry $(i, j)$ is seen only once. In
  addition, while both the runtime of our method and the
  \citet{CesaBianchiShamir2011} method is polynomial, the runtime of
  our method is significantly smaller: for $m \approx n$, each
  iteration of our method can be implemented in $\tilde{O}(n^3)$ time
  (see Section~\ref{sec:faster-alg}), whereas the runtime of each
  iteration in their algorithm is at least $\Omega(n^4)$ and can be
  significantly larger depending on the specific
  implementation.\footnote{Specifically, each iteration in their
    algorithm requires solving $n$ empirical risk minimization
    problems over the hypothesis space of $m \times n$ matrices with a
    bounded trace norm (in their notation, to obtain the optimal
    bound, one should set $T=n^2$ and $\eta \ge 1/n$, and then should
    solve $\eta T$ empirical risk minimization problems per
    iteration). It is not clear what is the optimal runtime of solving
    each such empirical risk minimization problem. We believe that it
    is impossible to obtain a solver which is significantly faster
    than $n^4$.}
\end{enumerate}

Finally, we derive (nearly) matching lower
  bounds for the three problems. In particular, our lower bound for
  the online collaborative filtering problem implies that the sample
  complexity of learning matrices with bounded entries and trace norm
  of $\Theta(n)$ is $\Omega(n^{3/2})$. This matches
  an upper bound on the sample complexity derived by \citet{ShamirSh11}
  and solves an open problem posed by \citet{ShamirSr11}. 

\section{Problem statements and main results}  \label{sec:main}

We start with the definition of $(\beta,\tau)$-decomposability. For
this, we first define a symmetrization operator.
\begin{Def}[Symmetrization]
	Given an $m \times n$ {\em non-symmetric} matrix $\bW$ its symmetrization is the $(m+n) \times (m+n)$ matrix:
$$ \sym(\bW) :=  \left[ \begin{array}{cc}
    \bzero & \bW \\
    \bW^\top & \bzero  \end{array} \right].$$
If $m = n$ and $\bW$ is symmetric, then $\sym(\bW) := \bW$. 
\end{Def}
The main property of matrices we rely on is $(\beta,
\tau)$-decomposability, which we define below. 
\begin{Def}[$(\beta, \tau)$-decomposability]
  An $m \times n$ matrix $\bW$ is $(\beta, \tau)$-decomposable if
  there exist symmetric, positive semidefinite matrices $\bP, \bN \in
  \mathbb{R}^{p \times p}$, where $p$ is the order of $\sym(\bW)$,
  such that the following conditions hold:
	\begin{align*}
		\sym(\bW)\ &=\ \bP - \bN,\\
		\Tr(\bP) + \Tr(\bN)\ &\leq\ \tau,\\
		\forall i \in [p]: P(i,i), N(i,i)\ &\leq\ \beta.
	\end{align*}
We say that a set of matrices $\mW$ is $(\beta,\tau)$-decomposable if
every matrix in $\mW$ is $(\beta,\tau)$-decomposable. 
\end{Def}
In the above, the parameter $\beta$ stands for a \textbf{b}ound on the
diagonal elements of $\bP$ and $\bN$, while the parameter $\tau$
stands for the \textbf{t}race of $\bP$ and $\bN$.  It is easy to
verify that if $\mW$ is $(\beta,\tau)$-decomposable then so is its
convex hull, $\conv(\mW)$. Throughout this paper, we assume for technical convenience that $\beta \geq 1$.\footnote{The condition $\beta \geq 1$ is not a serious restriction since for any $(\beta, \tau)$-decomposition of $\bW$, viz. $\sym(\bW) = \bP - \bN$, we have $\beta \geq |P(i, j)|, |N(i, j)|$  for all $(i, j)$ since $\bP, \bN \succeq \bzero$; and so $2\beta \geq |P(i, j) - N(i, j)| = |W(i, j)|$. Thus, if we make the reasonable assumption that there is some $\bW \in \mW$ with $|W(i, j)| = 1$ for some $(i, j)$, then $\beta \geq \frac{1}{2}$ is necessary.} 

There is an intriguing connection between the $(\beta, \tau)$-decomposition for a rectangular matrix $\bW$ and its max-norm and trace norm: the least possible $\beta$ in any $(\beta, \tau)$-decomposition exactly equals half the max-norm of $\bW$ (see Theorem~\ref{thm:max-norm-connection}), and the least possible $\tau$ in any $(\beta, \tau)$-decomposition exactly equals twice the trace-norm of $\bW$ (see Theorem~\ref{thm:trace-norm-connection}).

Our first contribution is a generic low regret algorithm for online
matrix prediction with a $(\beta,\tau)$-decomposable comparison
class. We also assume that all the matrices in the comparison class
have bounded entries.  Formally, we consider the following problem.

\begin{myalgo}{Online Matrix Prediction}
\textbf{parameters:} $\beta \ge 1$, $\tau \ge 0$, $G \ge 0$ \\
\textbf{input:} A set of matrices, $\mW \subset [-1,1]^{m \times n}$, 
which is $(\beta,\tau)$-decomposable \\
\textbf{for}~$t=1,2,\ldots,T$ \+ \\
adversary supplies a pair of indices $(i_t, j_t) \in [m] \times [n]$ \\
learner picks $\bW_t \in \conv(\mW)$ and outputs the prediction $W_t(i_t,
j_t)$ \\
adversary supplies a convex, $G$-Lipschitz, loss function $\ell_t:
[-1, 1] \rightarrow \mathbb{R}$ \\
learner pays $\ell_t(W_t(i_t,j_t))$
\end{myalgo}




\begin{Thm} \label{thm:Main}
There exists an efficient algorithm for Online Matrix Prediction which
enjoys the regret bound
\[
\text{Regret} ~\le~ 2G\sqrt{\tau \beta \log(2p) T} ,
\]
where $p$ is the order of $\sym(\bW)$ for any matrix $\bW \in \mW$.
\end{Thm}

The Online Matrix Prediction problem captures several specific
problems considered in the literature, given in the next few
subsections.

\subsection{Online Max-Cut} 

Recall that on each round of online max-cut, the learner should decide
whether two vertices of a graph, $(i_t,j_t)$ are joined by an edge or
not. The learner outputs a number $\hat{y}_t \in [-1, 1]$ which is to be interpreted as a randomized prediction in $\{-1, 1\}$: predict $1$ with probability $\frac{1 + \hat{y}_t}{2}$ and $-1$ with the remaining probability. The adversary then supplies
the true outcome, $y_t \in \{-1, 1\}$, where $y_t = 1$ 
indicates the outcome ``$(i_t, j_t)$ are joined by an edge'', and $y_t = -1$ the
opposite outcome. The loss suffered by the learner is the absolute
loss,
\[ \ell_t(\hat{y}_t)\ = \frac{1}{2} {|\hat{y_t} - y_t|},\] 
which can be also interpreted as the probability that a randomized prediction according to
$\hat{y}_t$ will not equal the true outcome $y_t$.

The comparison class is $\mW = \{\bW_A | A \subseteq [n]\}$,
where
$$ W_A(i,j) =  \begin{cases} 
1\ &\ \text{if } ((i \in A) \text{ and } (j \notin  A))   \text{ or }  ((j \in A) \text{ and } (i \notin  A))  \\
-1\ &\ \text{otherwise.}
\end{cases}
$$ 
That is, $W_A(i,j)$ indicates if $(i,j)$ crosses the cut defined by
$A$ or not. The following lemma (proved in Appendix~\ref{sec:max-cut-proof}) formalizes the relationship of this online problem
to the max-cut problem:
\begin{Lem} \label{lem:maxcutoffline}
  Consider an online sequence of loss functions $\{\ell_t\}$ as
  above. Let
$$ \bW^* = \arg \min_{\bW \in \mW}  \sum_t \ell_t(W(i_t,j_t)) ~. $$
Then $\bW^* = \bW_A$ for the set $A$ that determines the max cut in
the weighted graph over $[n]$ nodes whose weights are given by $w_{ij}
= \sum_{t : (i_t,j_t) = (i,j)} y_t$ for every $(i,j)$.
\end{Lem}

A regret bound of $O(\sqrt{n T})$ is attainable for this problem as
follows via an exponential time algorithm: consider the set of all $2^n$ cuts in the graph. For each cut
defined by $A$, consider a decision rule or ``expert'' that predicts
according to the matrix $\bW_A $. Standard bounds for the experts
algorithm imply the $O(\sqrt{n T})$ regret bound.


A simple way to get an efficient algorithm is to replace $\mW$ with the
class of all matrices in $\{-1, 1\}^{n \times n}$. This leads to $n^2$
different prediction tasks, each of which corresponds to the decision
if there is an edge between two nodes, which is efficiently
solvable. However, the regret with respect to this larger comparison
class scales like $O(\sqrt{n^2 T})$.

Another popular approach for circumventing the hardness is to replace
$\mW$ with the set of matrices whose trace-norm is bounded by $\tau =
n$. However, applying the online mirror descent algorithmic framework with an
appropriate squared-Schatten norm regularization, as described in
\citep{kakade2009duality}, leads to a regret bound that again scales
like $O(\sqrt{n^2 T})$.


In contrast, our Online Matrix Prediction algorithm yields an efficient
solution for this problem, with a regret that scales like $\sqrt{n
  \log(n) T}$. The regret bound of the algorithm follows from
the following:
\begin{Lem} \label{lem:max-cut-decomposable}
$\mW$  is $(1,n)$-decomposable. 
\end{Lem}
Combining the above with \thmref{thm:Main} yields:
\begin{Cor} \label{thm:max-cut-regret}
	There is an efficient algorithm for the online max-cut problem with regret bounded by $2\sqrt{n\log(n)T}$. 
\end{Cor}

We prove (in Appendix~\ref{sec:lower-bounds}) that the upper bound is near-optimal:
\begin{Thm} \label{thm:max-cut-LB}
	For any algorithm for the online max-cut problem, there is a sequence of entries $(i_t, j_t)$ and loss functions $\ell_t$ for $t = 1, 2, \ldots, T$ such that the regret of the algorithm is at least $\sqrt{{nT}/{16}}$.
\end{Thm}

\subsection{Collaborative Filtering with Bounded Trace Norm}

In this problem, the comparison set $\mW$ is the following set of $m
\times n$ matrices with trace norm bounded by some parameter $\tau$:
\begin{equation} \label{eqn:mWcf} \mW\ :=\ \{\bW \in [-1, 1]^{m \times n}:\ \|\bW\|_\star \leq
\tau\}.
\end{equation}
Without loss of generality we assume that $m \le n$.   

As before, applying the technique of \citet{kakade2009duality} leads to
a regret bound that scales as $\sqrt{\tau^2 T}$, which leads to
trivial results in the most relevant case where $\tau =
\Theta(\sqrt{mn})$. In contrast, we can obtain a much better result
based on the following lemma.
\begin{Lem} \label{lem:cf-decomposable}
The class $\mW$ given in (\ref{eqn:mWcf}) is
$(\sqrt{m+n},2\tau)$-decomposable. 
\end{Lem}
Combining the above with \thmref{thm:Main} yields:
\begin{Cor} \label{cor:collaborative-filtering-regret} There is an efficient
  algorithm for the online collaborative filtering problem with regret
  bounded by $2G\sqrt{2\tau \sqrt{n+m}\log(2(m+n))T})$, assuming that
  for all $t$ the loss function is $G$-Lipschitz. 
\end{Cor}

This upper bound is near-optimal, as we can also show (in Appendix~\ref{sec:lower-bounds}) the following lower bound on the regret:
\begin{Thm} \label{thm:collaborative-filtering-LB} For any algorithm for
  online collaborative filtering problem with trace norm bounded by $\tau$,
  there is a sequence of entries $(i_t, j_t)$ and $G$-Lipschitz loss functions
  $\ell_t$ for $t = 1, 2, \ldots, T$ such that the regret of the
  algorithm is at least  $G\sqrt{\frac{1}{2}\tau\sqrt{n}T}$.
\end{Thm}

In fact, the technique used to prove the above lower bound also
implies a lower bound on the sample complexity of collaborative filtering in
the batch setting (proved in Appendix~\ref{sec:lower-bounds}).
\begin{Thm} \label{thm:collaborative-filtering-batch-LB} The sample
  complexity of learning $\mW$ in the batch setting, is $\Omega(\tau
  \sqrt{n}/\epsilon^2)$. In particular, when $\tau = \Theta(n)$, the
  sample complexity is $\Omega({n^{1.5}}/\epsilon^2)$.
\end{Thm}
This matches an upper bound given by \citet{ShamirSh11}. The question
of determining the sample complexity of $\mW$ in the batch setting has
been posed as an open problem by Shamir (who conjectured that 
it scales like ${n^{1.5}}$) and Srebro (who conjectured that it scales
like ${n^{4/3}}$).

\subsection{Online gambling}

In the gambling problem, we define the comparison set $\mW$ as the
following set of $n \times n$ matrices. First, for every permutation
$\pi: [n] \rightarrow [n]$, define the matrix $\bW_\pi$ as:
\[ W_\pi(i, j)\ =\ \begin{cases} 
1\ &\ \text{if } \pi(i) \leq \pi(j)\\
0\ &\ \text{otherwise.}
\end{cases}
\]
Then the set $\mW$ is defined as
\begin{equation} \label{eqn:mWgamble} 
\mW\ :=\ \{\bW_\pi:\ \pi \text{ is a permutation of } [n]\}. 
\end{equation}
On round $t$, the adversary supplies a pair $(i_t, j_t)$ with $i_t
\neq j_t$, and the learner outputs as a prediction $\hat{y}_t =
W_t(i_t,j_t) \in [0, 1]$, where we interpret $\hat{y}_t$ as the
probability that $i_t$ will beat $j_t$. The adversary then supplies
the true outcome, $\hat{y}_t \in \{0, 1\}$, where $\hat{y}_t = 1$
indicates the outcome ``$i_t$ beats $j_t$'', and $\hat{y}_t = 0$ the
opposite outcome. The loss suffered by the learner is the absolute
loss,
\[ \ell_t(y_t)\ =\ |y_t - \hat{y}_t|,\] which can be also interpreted
as the probability that a randomized prediction according to
$\hat{y}_t$ will not equal to the true outcome $y_t$.

As before, we tackle the problem by analyzing the decomposability of
$\mW$.
\begin{Lem} \label{lem:gambling}
The class $\mW$ given in (\ref{eqn:mWgamble}) is
$(O(\log(n)),O(n\log(n)))$-decomposable. 
\end{Lem}
Combining the above with \thmref{thm:Main} yields:
\begin{Cor} \label{cor:gambling}
  There is an efficient algorithm for the online gambling problem with regret bounded by $O(\sqrt{n\log^3(n)T})$. 
\end{Cor}

This upper bound is near-optimal, as \citet{KNW}
essentially prove the following lower bound on the regret:
\begin{Thm}
  For any algorithm for the online gambling problem, there is a
  sequence of entries $(i_t, j_t)$ and labels $y_t$, for $t = 1, 2,
  \ldots, T$, such that the regret of the algorithm is at least
  $\Omega(\sqrt{n \log(n) T})$.
\end{Thm}

\section{The Algorithm for Online Matrix Prediction}
\label{sec:OCO}

In this section we prove \thmref{thm:Main} by constructing an
efficient algorithm for Online Matrix Prediction and analyze its
regret. We start by describing an algorithm for Online Linear
Optimization (OLO) over a certain set of matrices and with a certain set of
linear loss functions. We show later that the Online Matrix Prediction
problem can be reduced to this online convex optimization problem. 

\subsection{The $(\beta,\tau,\gamma)$-OLO problem} 

In this section, all matrices are in the space of real symmetric matrices of size $N \times
N$, which we denote by $\sS^{N \times N}$. 

On each round of online linear optimization, the learner chooses an
element from a convex set $\K$ and the adversary responds with a
linear loss function. In our case, the convex set $\K$ is a subset of the set of matrices with bounded trace and diagonal values:
\[ \K\ \subseteq\ \{\bX \in \sS^{N \times N}:\ \bX \succeq \bzero,\
\forall i \in [N]:\ X_{ii} \leq \beta,\ \Tr(\bX) \leq \tau\}. \] We
assume for convenience that $\frac{\tau}{N}\eye \in \K$.  The loss
function on round $t$ is the function $\bX \mapsto \bX \bullet \bL_t
\stackrel{\mathrm{def}}{=} \sum_{i,j} X(i,j)L_t(i,j)$, where $\bL_t$ is a matrix from
the following set of matrices:
\[ \mL\ =\ \{\bL \in \sS^{N \times N} :\ \bL^2 \stackrel{\mathrm{def}}{=}\bL\bL \text{ is a diagonal
  matrix s.t. } \Tr(\bL^2) \leq \gamma\}.\]
We call the above setting a $(\beta, \gamma, \tau)$-OLO problem.

As usual, we analyze the regret of the algorithm
\[
\text{Regret}:= \sum_{t=1}^T \bX_t \bullet \bL_t - \min_{\bX \in \K} \sum_{t=1}^T \bX \bullet \bL_t ~,
\]
where $\bX_1,\ldots,\bX_T$ are the predictions of the learner. 

Below we describe and analyze an algorithm for the
$(\beta,\gamma,\tau)$-OLO problem. The algorithm, forms of which independently appeared in the work of \citet{TRW} and \citet{AroraKale}, performs exponentiated gradient steps followed by
Bregman projections onto $\K$. The projection operation is defined with
respect to the quantum relative entropy divergence:
\[
\Delta(\bX,\bA) = \Tr( \bX \log(\bX) - \bX \log(\bA) - \bX + \bA).
\]
\begin{algorithm}[H]
\caption{Matrix Multiplicative Weights with Quantum Relative Entropy Projections} \label{mmw-alg}
\begin{algorithmic} [1]
\STATE Input: $\eta$
\STATE Initialize $\bX_1 = \frac{\tau}{N}\eye$.
\FOR{$t=1,2, \ldots, T$:}
\STATE Play the matrix $\bX_t$.
\STATE Obtain loss matrix $\bL_t$.
\STATE \label{eq:X-update} Update $\bX_{t+1} = \arg\min_{\bX \in \K}\Delta(\bX,\ \exp(\log(\bX_t) - \eta \bL_t))$.
\ENDFOR
\end{algorithmic}
\end{algorithm}

Algorithm~\ref{mmw-alg} has the following regret bound (essentially following \citet{TRW,
  AroraKale}, also proved in Appendix~\ref{sec:mmw-analysis} for
completeness):
\begin{Thm} \label{thm:mmw-regret} Suppose $\eta$ is chosen so that
  $\eta\|\bL_t\| \leq 1$ for all $t$ (where $\|\bL_t\|$ is the
  spectral norm of $\bL_t$). Then
  \[ \text{Regret}\ \leq\ \eta \sum_{t=1}^T \bX_t \bullet \bL_t^2 +
  \frac{\tau\log(N)}{\eta}. \]
\end{Thm}

Equipped with the above we are ready to prove a regret bound for
$(\beta, \gamma, \tau)$-OLO. 

\begin{Thm} \label{thm:oco-regret} Assume $T \geq
  \frac{\tau\log(N)}{\beta}$. Then, applying Algorithm~\ref{mmw-alg}
  with $\eta = \sqrt{\frac{\tau\log(N)}{\beta \gamma T}}$ on a
  $(\beta, \gamma, \tau)$-OLO problem yields an efficient algorithm whose
  regret is at most $2\sqrt{\beta \gamma \tau \log(N) T}$.
\end{Thm}
\begin{proof}
  Clearly, Algorithm~\ref{mmw-alg}) can be implemented in polynomial
  time since the update of step~\ref{eq:X-update} is a convex optimization problem. To analyze the regret of the algorithm we rely on
  \thmref{thm:mmw-regret}. By the definition of $\K$ and $\mL$, we get
  that $\bX_t \bullet \bL_t^2 \leq \beta \gamma$.  Hence, the regret
  bound becomes
  \[ \text{Regret}\ \leq\ \eta \beta \gamma T +
  \frac{\tau\log(N)}{\eta}.\] Substituting the value of $\eta$, we get
  the stated regret bound. One technical condition is that the above
  regret bound holds as long as $\eta$ is chosen small enough so that
  for all $t$, we have $\eta\|\bL_t\| \leq 1$. Now $\|\bL_t\| \leq
  \|\bL_t\|_F = \sqrt{\Tr(\bL_t^2)} \leq \sqrt{\gamma}$.  Thus, for $T
  \geq \frac{\tau\log(N)}{\beta}$, the technical condition is
  satisfied for $\eta = \sqrt{\frac{\tau\log(N)}{\beta \gamma T}}$.
\end{proof}

\subsection{An Algorithm for the Online Matrix Prediction Problem} \label{sec:oco}

\def\bWhat{\sym(\bW)}
\def\phiinv{\phi^{-1}}

In this section we describe a reduction from the Online Matrix
Prediction problem (with a $(\beta,\tau)$-decomposable comparison
class) to a $(\beta, 4G^2, \tau)$-OCO problem with $N = 2p$. The
regret bound of the derived algorithm will follow directly from
\thmref{thm:oco-regret}.

We now describe the reduction. To simplify our notation, let $q$ be
$m$ if $\mW$ contains non-symmetric matrices and $q=0$ otherwise. Note
that the definition of $\sym(\bW)$ implies that for a pair of indices
$(i, j) \in [m] \times [n]$, their corresponding indices in
$\sym(\bW)$ are $(i,j+q)$.

Given any matrix $\bW \in \mW$ we embed its symmetrization $\bWhat$
(which has size $p \times p$) into the set of $2p \times 2p$ positive
semidefinite matrices as follows. Since $\bW$ admits a $(\beta,
\tau)$-decomposition, there exist $\bP, \bN \succeq \bzero$ such that
$\bWhat = \bP - \bN$, $\Tr(\bP) + \Tr(\bN) \leq \tau$, and for all $i
\in [p]$, $P(i,i), N(i,i) \leq \beta$. The embedding of $\bW$ in
$\sS^{2p \times 2p}$, denoted $\phi(\bW)$, is defined to be the
matrix\footnote{Note that this mapping depends on the choice of $\bP$
  and $\bN$ for each matrix $\bW \in \mW$. We make an arbitrary choice
  for each $\bW$.}
\[\phi(\bW) ~=~ \left[
    \begin{array}{cc}
      \bP & \bzero \\
      \bzero & \bN \\
    \end{array}
  \right].\] 
It is easy to verify that $\phi(\bW)$ belongs to the convex set $\K$
defined below: 

\begin{align}
	\K:=\ \Biggl\{&~~\bX \in  \sS^{2p \times 2p}\ \text{ s.t. } \notag \\
&\bX \succeq \bzero \label{eq:valid-pred}\\
&	\forall i \in [2p]:\ X(i, i)\ \leq\ \beta \notag \\
&	\Tr(\bX)\ \leq\ \tau \notag \\
&	\forall (i, j) \in [m] \times [n]:\ (X(i, j+q) - X(p+i,p+j+q))\ \in [-1, 1] 
~~\Biggr\} \notag
\end{align}

We shall run the OLO algorithm with the set $\K$. On round $t$, if the
adversary gives the pair $(i_t, j_t)$, then we predict
\[
\hat{y}_t = X_t(i_t, j_t+q) - X_t(p+i_t,p+j_t+q) ~.
\]
The last constraint defining $\K$ simply ensures that $\hat{y}_t \in
[-1, 1]$. While this constraint makes the quantum relative entropy
projection onto $\K$ more complex, in Appendix~\ref{sec:faster-alg} we
show how we can leverage the knowledge of $(i_t, j_t)$ to get a very
fast implementation.

Next we describe how to choose the loss matrices $\bL_t$ using the
subderivative of $\ell_t$. Given the loss function $\ell_t$, let $g$
be a subderivative of $\ell_t$ at $\hat{y}_t$. Since $\ell_t$ is
convex and $G$-Lipschitz, we have that $|g| \le G$. Define $\bL_t \in
\sS^{2p \times 2p}$ as follows:
\begin{equation} \label{eqn:Ltdef}
L_t(i,j) = \begin{cases}
g ~~& \textrm{if}~ (i,j)=(i_t,j_t+q) \text{ or } (i,j) = (j_t+q,i_t) \\
-g & \textrm{if}~ (i,j)=(p+i_t,p+j_t+q) \text{ or } (i,j) = (p+j_t+q,p+i_t)
\\
0 & \textrm{otherwise}.
\end{cases}
\end{equation}
Note that $\bL_t^2$ is a diagonal matrix, whose only non-zero diagonal
entries are $(i_t+q, i_t+q)$, $(j_t+q, j_t+q)$, $(p+i_t+q, p+i_t+q)$,
and $(p+j_t+q, p+j_t+q)$, all equalling $g^2$. Hence, $\Tr(\bL_t^2) =
4g^2 \leq 4G^2$.

To summarize, the Online Matrix Prediction algorithm will be as follows:
\begin{algorithm}[H]
\caption{Matrix Multiplicative Weights for Online Matrix Prediction} \label{mmw-alg-OMP}
\begin{algorithmic} [1]
\STATE Input: $\beta,\tau,G,m,n,p,q$ (see text for definitions)
\STATE Set: $\gamma = 4G^2$, $N=2p$, $\eta =
\sqrt{\frac{\tau\log(N)}{\beta \gamma T}}$
\STATE Let $\K$ be as defined in (\ref{eq:valid-pred})
\STATE Initialize $\bX_1 = \frac{\tau}{N}\eye$.
\FOR{$t=1,2, \ldots, T$:}
\STATE Adversary supplies a pair of indices $(i_t, j_t) \in [m] \times [n]$.
\STATE Predict $\hat{y}_t = X_t(i_t, j_t+q) - X_t(p+i_t,p+j_t+q)$.
\STATE Obtain loss function $\ell_t : [-1,1] \to \reals$ and pay $\ell_t(\hat{y}_t)$.
\STATE Let $g$ be a sub-derivative of $\ell_t$ at $\hat{y}_t$
\STATE Let $\bL_t$ be as defined in (\ref{eqn:Ltdef})
\STATE Update $\bX_{t+1} = \arg\min_{\bX \in \K}\Delta(\bX,\ \exp(\log(\bX_t) - \eta \bL_t))$.
\ENDFOR
\end{algorithmic}
\end{algorithm}

To analyze the algorithm, note that for any $\bW \in \mW$,
\[ \phi(\bW) \bullet \bL_t = 2g(P(i_t,j_t)-N(i_t,j_t)) = 2g W(i_t, j_t), \]
and \[ \bX_t \bullet \bL_t\ =\ 2g (X_t(i_t, j_t+q) - X_t(p+i_t, p+j_t+q))\ =\ 2g\hat{y}_t. \]
So for any $\bW \in \mW$, we have
\begin{align*}
	\bX_t \bullet \bL_t - \phi(\bW) \bullet \bL_t\ &=\ 2g(\hat{y}_t - W(i_t, j_t))\\
	&\geq\ 2(\ell_t(\hat{y}_t) - \ell_t(W(i_t, j_t))), 
\end{align*}
by the convexity of $\ell_t(\cdot)$. This implies that for any $\bW \in \mW$,
\[ \sum_{t=1}^T \ell_t(\hat{y}_t) - \ell_t(W(i_t, j_t))\ \leq\
\frac{1}{2}\left[\sum_{t=1}^T \bX_t \bullet \bL_t - \phi(\bW) \bullet
  \bL_t\right]\ \leq\ \frac{1}{2} \cdot \text{Regret}_\text{OLO}.\]
Thus, the regret of the Online Matrix Prediction problem is at most half the regret in the
$(\beta, 4G^2, \tau)$-OLO problem.

\subsubsection{Proof of \thmref{thm:Main}}
Following our reduction, we can now appeal to
Theorem~\ref{thm:oco-regret}. For $T \geq \frac{\tau
  \log(2p)}{\beta}$, the bound of Theorem~\ref{thm:oco-regret} applies
and gives a regret bound of $2G\sqrt{\tau \beta \log(2p) T}$. For $T <
\frac{\tau \log(2p)}{\beta}$, note that in any round, the regret can
be at most $2G$, since the subderivatives of the loss functions are
bounded in absolute value by $G$ and the domain is $[-1, 1]$, so the
regret is bounded by $2GT < 2G\sqrt{\tau \beta \log(2p) T}$ since
$\beta \geq 1$. Thus, we have proved the regret bound stated in
\thmref{thm:Main}.

\section{Decomposability Proofs} 

In this section we prove the decomposability results for the
comparison classes corresponds to max-cut, collaborative filtering, and
gambling. All the three decompositions we give are optimal up to constant factors.

\subsection{Proof of \lemref{lem:max-cut-decomposable} (max-cut)}

We need to show that every matrix $\bW_A \in \mW$ admits a
$(1,n)$-decomposition. We can rewrite $\bW_A = - \bw_A \bw^\top$ where
$\bw_A \in \reals^n$ is the vector such that
$$ W_A(i) = \begin{cases} 
1\ &\ \text{if } i \in A  \\
-1\ &\ \text{otherwise.}
\end{cases}
$$ 
Since $\bW_A$ is already symmetric, $\sym(\bW_A) = \bW_A = -\bw_A
\bw_A^\top$. Thus we can choose $\bP=\bzero$ and $\bN=\bw_A
\bw_A^\top$. These are positive semidefinite matrices with diagonals
bounded by $1$ and sum of traces equals to $n$, which concludes the proof. Since $\Tr(\bw_A\bw_A^\top) = n$, this
$(1, n)$-decomposition is optimal.

\subsection{Proof of \lemref{lem:cf-decomposable} (collaborative filtering)}

We need to show that every matrix $\bW \in \mW$, i.e. an $m \times n$
matrix over $[-1,1]$ with $\|\bW\|_\star \leq \tau$, admits a
$(\sqrt{m+n}, 2\tau)$-decomposition. The $(\sqrt{m+n},
2\tau)$-decomposition of $\bW$ is a direct consequence of the
following theorem, setting $\bY = \sym(\bW)$, with $p = m + n$, and
the fact that $\|\sym(\bW)\|_\star = 2\|\bW\|_\star$ (see \lemref{lem:sym}).
\begin{Thm} \label{thm:matrix-decomp}
Let $\bY$ be a $p \times p$ symmetric matrix with entries in $[-1, 1]$. Then 
$\bY$ can be written as $\bY = \bP - \bN$ where $\bP$ and $\bN$ are both 
positive semidefinite matrices with diagonal entries bounded by $\sqrt{p}$, 
and $\Tr(\bP) + \Tr(\bN) = \|\bY\|_\star$.
\end{Thm}
\begin{proof}
Let
\[\bY = \sum_i \lambda_i \bv_i\bv_i^\top\]
be the eigenvalue decomposition of $\bY$. We now show that
\[\bP = \sum_{i:\ \lambda_i \geq 0} \lambda_i \bv_i\bv_i^\top \text{ and } \bN 
= \sum_{i:\ \lambda_i < 0} -\lambda_i \bv_i\bv_i^\top\] satisfy the required 
conditions. Clearly $\Tr(\bP) + \Tr(\bN) = \sum_i |\lambda_i| = 
\|\bY\|_\star$. Define $\abs(\bY) = \bP + \bN = \sum_i 
|\lambda_i|\bv_i\bv_i^\top$. Note that 
\[\abs(\bY)^2\ =\ \sum_i \lambda_i^2 \bv_i\bv_i^\top\ =\ \bY^2.\]
We now show that {\em all} entries (and in particular, the diagonal
entries) of $\abs(\bY)$ are bounded in magnitude by $\sqrt{p}$. Since
$\bP$ and $\bN$ are both positive semidefinite, their diagonal
elements must be non-negative, so we conclude that the diagonal
entries of $\bP$ and $\bN$ are bounded by $\sqrt{p}$ as well.

Since all the entries of $\bY$ are bounded in magnitude by $1$, it follows 
that all entries of $\bY^2$ are bounded in magnitude by $p$. In particular, 
the diagonal entries of $\bY^2$ are bounded by $p$. Since these diagonal 
entries are equal to the squared lengths of the rows of $\abs(\bY)$, it 
follows that each entry of $\abs(\bY)$ is bounded in magnitude by $\sqrt{p}$.
\end{proof}

This decomposition is optimal up to constant factors. Consider the
matrix $\bW$ formed by taking $m = \frac{\tau}{\sqrt{n}}$ rows of an
$n \times n$ Hadamard matrix. In Theorem~\ref{thm:ocf-decomp-optimal} (proved in Appendix~\ref{sec:cf-optimality}),
we prove that any $(\beta, \tilde{\tau})$-decomposition of $\sym(\bW)$
must have $\beta\tilde{\tau} \geq \frac{1}{4}\tau\sqrt{n}$. Since the
regret bound depends on the product $\beta \tilde{\tau}$, we conclude
that the decomposition obtained from \thmref{thm:matrix-decomp} is optimal up to a constant factor.

\subsection{Proof of \lemref{lem:gambling} (gambling)}

We need to show that every matrix $\bW \in \mW$, i.e. an $n \times n$
matrix $\bW_\pi$ for some permutation $\pi:[n] \rightarrow [n]$,
admits a $(O(\log(n)), O(n\log(n)))$-decomposition. One minor change
that needs to be made to Algorithm \ref{mmw-alg-OMP} is that the last
constraint in (\ref{eq:valid-pred}) needs to be changed to
\[ \forall (i, j) \in [n] \times [n]:\ (X(i, j+q) - X(p+i,p+j+q))\ \in [0, 1], \]
to ensure that the prediction lies in $[0, 1]$ rather than $[-1,
1]$. The analysis remains intact, and so does the regret bound.

We now give the decomposition. The following upper triangular matrix $\bT$ plays a pivotal role:
\[ T(i, j) = \begin{cases}
	1\ &\text{ if } i \leq j\\
	0\ &\text{ otherwise.}
\end{cases}\]
The reason this matrix is so important is because any matrix $\bW_\pi$ is obtained by permuting the rows and columns of $\bT$. In particular, let $\bP_\pi$ be the permutation matrix defined by the permutation $\pi$, i.e.
\[
P_\pi(i, j) = \begin{cases}
	1\ &\text{ if } j = \pi(i)\\
	0\ &\text{ otherwise.}
\end{cases}
\]
Then it is easy to check that
\[ \bW_\pi = \bP_\pi \bT \bP_\pi^\top. \]
Using this fact, we get
\begin{align*}
\underbrace{\left[ \begin{array}{ccc}
\bP_\pi & \bzero \\
\bzero & \bP_\pi  \end{array} \right]}_{\bQ_\pi} \sym(\bT)  \left[ \begin{array}{ccc}
\bP_\pi^\top & \bzero \\
\bzero & \bP_\pi^\top  \end{array} \right]\ &=\ \left[ \begin{array}{ccc}
\bP_\pi & \bzero \\
\bzero & \bP_\pi  \end{array} \right] \left[ \begin{array}{ccc}
\bzero & \bT \\
\bT^\top & \bzero \end{array} \right] \left[ \begin{array}{ccc}
\bP_\pi^\top & \bzero \\
\bzero & \bP_\pi^\top  \end{array} \right]
\\
&=\ \left[ \begin{array}{ccc}
\bzero & \bP_\pi \bT \bP_\pi^\top \\
\bP_\pi \bT^\top \bP_\pi^\top & \bzero \end{array} \right]\\
&=\ \left[ \begin{array}{ccc}
\bzero & \bW_\pi \\
\bW_\pi^\top & \bzero \end{array} \right]\ =\ \sym(\bW_\pi).
\end{align*}
Now, note that $\bQ_\pi$ is a permutation matrix (viz. the one defined by the permutation $\pi':[2n] \rightarrow [2n]$ defined as $\pi'(i) = \pi(i)$ for $1 \leq i \leq n$, and $\pi'(i) = \pi(i - n) + n$ for $n < i \leq 2n$). Thus, if $\bT$ admits a $(\beta, \tau)$-decomposition, $\sym(\bT) = \bP - \bN$, then 
\[ \sym(\bW_\pi)\ =\ \bQ_\pi \sym(\bT)\bQ_\pi^\top\ =\ \bQ_\pi\bP\bQ_\pi^\top - \bQ_\pi \bN\bQ_\pi^\top \]
is a $(\beta, \tau)$-decomposition for $\sym(\bW_\pi)$. This is because the diagonal entries of $\bQ_\pi\bP\bQ_\pi^\top$ (resp. $\bQ_\pi\bN\bQ_\pi^\top$) are simply a permutation (viz. $\pi'$) of the diagonal entries of $\bP$ (resp. $\bN$). Since $\bA\bB\bA^\top \succeq \bzero$ if $\bB \succeq \bzero$ for any matrix $\bA$, the matrices $\bQ_\pi\bP\bQ_\pi^\top$ and $\bQ_\pi\bP\bQ_\pi^\top$ are both positive semidefinite.

So now we show that $\bT$ admits a $(O(\log(n)), O(n\log(n)))$-decomposition. For convenience, we assume that $n$ is a power of $2$, i.e. $n = 2^k$ for some integer $k \geq 0$. For $n$ that are not a power of $2$, we can readily obtain a decomposition by the following observation: if we take the smallest power of $2$ that is larger than $n$, say $2^k$, and consider the symmetrized triangular matrix for $2^k$, then $\sym(\bT)$ can be expressed as a principal submatrix of it. Then taking the corresponding principal submatrices from the decomposition for the triangular matrix for $2^k$ we obtain a decomposition for $n$. This uses the fact that principal submatrices of positive semidefinite matrices are positive semidefinite as well. 

\begin{Thm}
	Let $n = 2^k$ for some integer $k \geq 0$. Then
	$\bT$ admits a $(k+1, 4n(k+1))$-decomposition.
\end{Thm}
\begin{proof}
	We show that $\sym(\bT)$ can be written as a difference of positive semidefinite matrices with diagonals bounded by $k+1$. The bound on the sum of traces, $4n(k+1)$, of the two matrices follows trivially.
		
	We use a recursive construction. Let the triangular matrix for $n = 2^k$ be denoted by $\bT_k$. For $k = 0$, the following is a decomposition for $\bT_0$ with diagonals bounded by $1$:
	\[ \sym(\bT_0) = \left[
    \begin{array}{cc}
      0 & 1 \\
      1 & 0
    \end{array}
  \right] = \left[
    \begin{array}{cc}
      1 & 1 \\
      1 & 1
    \end{array}
  \right]  - \left[
    \begin{array}{cc}
      1 & 0 \\
      0 & 1
    \end{array}
  \right] .\]
  So now assume that $k > 0$ and we have a decomposition for $\bT_{k-1}$ with diagonals bounded by $k$, i.e. 
  \[ \sym(\bT_{k-1})\ =\ \left[
    \begin{array}{cc}
      \bzero & \bT_{k-1} \\
      \bT_{k-1}^\top & \bzero
    \end{array}
  \right]\ =\ \bP - \bN, \]
  where $\bP, \bN \succeq \bzero$, and for all $i \in [2^k]$, $P(i,i), N(i, i) \leq k$.
  We need the following block decomposition of $\bP$ and $\bN$ into contiguous $2^{k-1} \times 2^{k-1}$ blocks as follows:
\[\bP = \left[
    \begin{array}{cc}
      \bP^A & \bP^B \\
      \bP^C & \bP^D
    \end{array}
  \right] \text{ and } \bN = \left[
    \begin{array}{cc}
      \bN^A & \bN^B \\
      \bN^C & \bN^D
    \end{array}
  \right].\]
  Then we have the following decomposition of $\sym(\bT_k)$. All the blocks in the decomposition below are of size $2^{k-1} \times 2^{k-1}$.
  \[
  \sym(\bT_k) = \left[
    \begin{array}{cccc}
      \bzero & \bzero & \bzero & \bone \\
      \bzero & \bzero & \bzero & \bzero \\
      \bzero & \bzero & \bzero & \bzero \\
      \bone & \bzero & \bzero & \bzero 
    \end{array}
  \right] +  \left[
    \begin{array}{cccc}
      \bzero & \bzero & \bT_{k-1} & \bzero \\
      \bzero & \bzero & \bzero & \bT_{k-1} \\
      \bT_{k-1}^\top & \bzero & \bzero & \bzero \\
      \bzero & \bT_{k-1}^\top & \bzero & \bzero 
    \end{array}
  \right].\]
  Now, consider the following decompositions of the two matrices above as a difference of positive semidefinite matrices. For the first matrix, the diagonals in the decomposition are bounded by $1$:
  \[ \left[
    \begin{array}{cccc}
      \bzero & \bzero & \bzero & \bone \\
      \bzero & \bzero & \bzero & \bzero \\
      \bzero & \bzero & \bzero & \bzero \\
      \bone & \bzero & \bzero & \bzero 
    \end{array}
  \right]  = \left[
    \begin{array}{cccc}
      \bone & \bzero & \bzero & \bone \\
      \bzero & \bzero & \bzero & \bzero \\
      \bzero & \bzero & \bzero & \bzero \\
      \bone & \bzero & \bzero & \bone 
    \end{array}
  \right]  - \left[
    \begin{array}{cccc}
      \bone & \bzero & \bzero & \bzero \\
      \bzero & \bzero & \bzero & \bzero \\
      \bzero & \bzero & \bzero & \bzero \\
      \bzero & \bzero & \bzero & \bone 
    \end{array}
  \right]. \]
  For the second matrix, the diagonals in the decomposition are bounded by $k$. 
  \[
  \left[
    \begin{array}{cccc}
      \bzero & \bzero & \bT_{k-1} & \bzero \\
      \bzero & \bzero & \bzero & \bT_{k-1} \\
      \bT_{k-1}^\top & \bzero & \bzero & \bzero \\
      \bzero & \bT_{k-1}^\top & \bzero & \bzero 
    \end{array}
  \right] = \left[
    \begin{array}{cccc}
      \bP^A & \bzero & \bP^B & \bzero \\
      \bzero & \bP^A & \bzero & \bP^B \\
      \bP^C & \bzero & \bP^D & \bzero \\
      \bzero & \bP^C & \bzero & \bP^D 
    \end{array}
  \right] - \left[
    \begin{array}{cccc}
      \bN^A & \bzero & \bN^B & \bzero \\
      \bzero & \bN^A & \bzero & \bN^B \\
      \bN^C & \bzero & \bN^D & \bzero \\
      \bzero & \bN^C & \bzero & \bN^D
    \end{array}
  \right].
  \]
  It is easy to verify that the matrices in the decomposition above are positive semidefinite, since each is a sum of two positive semidefinite matrices. For example:
  \[
	\left[
    \begin{array}{cccc}
      \bP^A & \bzero & \bP^B & \bzero \\
      \bzero & \bP^A & \bzero & \bP^B \\
      \bP^C & \bzero & \bP^D & \bzero \\
      \bzero & \bP^C & \bzero & \bP^D 
    \end{array}
  \right] = \left[
    \begin{array}{cccc}
      \bP^A & \bzero & \bP^B & \bzero \\
      \bzero & \bzero & \bzero & \bzero \\
      \bP^C & \bzero & \bP^D & \bzero \\
      \bzero & \bzero & \bzero & \bzero 
    \end{array}
  \right] + \left[
    \begin{array}{cccc}
      \bzero & \bzero & \bzero & \bzero \\
      \bzero & \bP^A & \bzero & \bP^B \\
      \bzero & \bzero & \bzero & \bzero  \\
      \bzero & \bP^C & \bzero & \bP^D      
    \end{array}
  \right].
  \]
  Adding the two decompositions, we get a decomposition for $\sym(\bT_k)$ as a difference of two positive semidefinite matrices. The diagonal entries of these two matrices are bounded by $k + 1$, as required.
\end{proof}

This decomposition is optimal up to constant factors. This is because the singular values of $\bT$ are $\frac{1}{2\cos (\frac{k\pi}{2n+1})}$ for $k = 1, 2, \ldots, n$ (see \citet{mathoverflow}). This implies that $\|\bT\|_\star = \Theta(n \log(n))$. Thus, the best $\beta$ one can get is $\Theta(\log(n))$, and the best $\tau$ is $\Theta(n \log(n))$.


\section{Lower bounds}
\label{sec:lower-bounds}

In this section we prove the lower bounds stated in \secref{sec:main}.

\subsection{Online Max Cut}

We prove \thmref{thm:max-cut-LB}, which
we restate here for convenience:
\begin{itemize}
\item[] \textbf{\thmref{thm:max-cut-LB} restated:}
For any algorithm for the online max cut problem, there is a
  sequence of entries $(i_t, j_t)$ and loss functions $\ell_t$ for $t
  = 1, 2, \ldots, T$ such that the regret of the algorithm is at least
  $\sqrt{{nT}/{16}}$.
\end{itemize}
\begin{proof}
  Consider the following stochastic adversary. Divide up the time
  period $T$ into $n/2$ equal size\footnote{We assume for convenience
    that $\frac{n}{2}$ and $\frac{2T}{n}$ are integers.} intervals
  $T_i$, for $i \in [n/2]$, corresponding to the $n/2$ pairs of
  indices $(i, i + n/2)$ for $i \in [n/2]$. For every $i \in [n/2]$
  and for each $t \in T_i$, the adversary sets $(i_t, j_t) = (i, i +
  n/2)$ and $y_t$ to be a Rademacher random variable independent of
  all other such variables. Clearly, the expected regret of any
  algorithm for the online max cut problem equals $\frac{T}{2}$.

  Now, define the following subset of vertices $A$: for every $i \in
  [n/2]$, consider $S_i = \sum_{t \in T_i} y_t$. If $S_i < 0$, include
  both $i, i+n/2 \in A$, else only include $i \in A$. By construction,
  the matrix $\bW_A$ has the following property for all $i \in [n/2]$:
  \[ W_A(i, i+n/2)\ =\ \sgn(S_i).\] 
Using the definition of $\ell_t$ and the fact that $|T_i| = 2T/n$, we obtain
\begin{align*} 
\E\left[\sum_{t \in T_i} \ell_t(W_A(i, i+n/2))\right]\ &=\ 
\E\left[\sum_{t \in T_i}
    \left(\thalf - \tfrac{\sgn(S_i)}{2} y_t\right)\right] \\
 &=\ \E\left[\frac{T}{n} - \frac{
      |S_i|}{2}\right]\ \leq\ \frac{T}{n} - \sqrt{\frac{T}{4n}},
\end{align*}
where we used Khintchine's inequality: if $X$ is a sum of $k$
independent Rademacher random variables, then $\E[|X|] \geq
\sqrt{k/2}$. Summing up over all $i \in [n/2]$, we get that
\[\E\left[\sum_{t=1}^T \ell_t(W_A(i_t, j_t))\right]\ \leq\ \frac{n}{2}\left[\frac{T}{n} - \sqrt{\frac{T}{4n}}\right]\ =\ \frac{T}{2} - \sqrt{\frac{nT}{16}}.\]
Hence the expected regret of the algorithm is at least
$\sqrt{\frac{nT}{16}}$. In particular, there is a setting of the
$\hat{y}_t$ variables so that the regret of the algorithm is at least
$\sqrt{\frac{nT}{16}}$.
\end{proof}

\subsection{Online Collaborative Filtering with Bounded Trace Norm}

We start with the proof of \thmref{thm:collaborative-filtering-LB}, which
we restate here for convenience:
\begin{itemize}
\item[] \textbf{\thmref{thm:collaborative-filtering-LB} restated:}
  For any algorithm for online collaborative filtering problem with trace
  norm bounded by $\tau$, there is a sequence of entries $(i_t, j_t)$
  and loss functions $\ell_t$ for $t = 1, 2, \ldots, T$ such that the
  regret of the algorithm is at least
  $G\sqrt{\frac{1}{2}\tau\sqrt{n}T}$.
\end{itemize}
\begin{proof}
First, we may assume that $\tau \leq m\sqrt{n}$: this is because for any matrix $\bW \in [-1, 1]^{m \times n}$,
we have 
\[ \|\bW\|_\star\ \leq\ \sqrt{\text{rank}(\bW)}\|\bW\|_F\ \leq\
\sqrt{m} \cdot \sqrt{mn} = m\sqrt{n},\] since $\text{rank}(\bW) \leq
m$. So now we focus on the sub-matrix formed by the first
$\frac{\tau}{\sqrt{n}}$ rows\footnote{For convenience, we assume that
  $\frac{\tau}{\sqrt{n}}$ and $\frac{T}{\tau\sqrt{n}}$ are integers.}
and all $n$ columns. This sub-matrix has $\tau\sqrt{n}$ entries.

Consider the following stochastic adversary. Divide up the time period
$T$ into $\tau\sqrt{n}$ intervals of length $\frac{T}{\tau\sqrt{n}}$,
indexed by $\tau\sqrt{n}$ pairs $(i, j)$ corresponding to the entries
of the sub-matrix. For every $(i, j)$, and for every round $t$ in the
interval $I_{ij}$ corresponding to $(i, j)$, we set the loss function
to be $\ell_t(\bW) = \sigma_tGW_{ij}$, where $\sigma_t \in \{-1, 1\}$
is a Rademacher random variable chosen independently of all other such
variables. Note that the absolute value of derivative of the loss
function is $G$.

Clearly, any algorithm for OCF has expected loss $0$. Now consider the
matrix $\bW^\star$ where
\[\forall i \in \left[\frac{\tau}{\sqrt{n}}\right], j \in [n]:\ W^\star_{ij}\ =\ -\sgn\left(\tsum_{t \in I_{ij}} \sigma_t\right),\]
and all entries in rows $i > \frac{\tau}{\sqrt{n}}$ are set to $0$.
Since $\text{rank}(\bW^\star) \leq \frac{\tau}{\sqrt{n}}$, we have
\[\|\bW^\star\|_\star\ \leq\ \sqrt{\text{rank}(\bW^\star)}\cdot
\|\bW^\star\|_F\ \leq\ \sqrt{\frac{\tau}{\sqrt{n}}} \cdot
\sqrt{\tau\sqrt{n}}\ =\ \tau,\] so $\bW^\star \in \mW$.\footnote{This
  construction is tight: e.g. if $\bW^\star$ is formed by taking
  $\frac{t}{\sqrt{n}}$ rows of an $n \times n$ Hadamard matrix.}

The expected loss of $\bW^\star$ is
\begin{align*}
	\sum_{ij}\av\left[\sum_{t \in I_{ij}} \sigma_t GW^\star_{ij}\right]\ &=\ 
G\sum_{ij}\av\left[-\left|\sum_{t \in I_{ij}} \sigma_t\right|\right]\\
&\geq\ -G\sum_{ij} \sqrt{\frac{1}{2}|I_{ij}|}\\
&=\ -G\tau\sqrt{n} \cdot \sqrt{\frac{T}{2\tau\sqrt{n}}}\\
&=\ -G\sqrt{\frac{1}{2}\tau\sqrt{n}T},
\end{align*}
where the inequality above is again due to Khintchine's
inequality. Hence, the expected regret of the algorithm is at least
$G\sqrt{\frac{1}{2}\tau\sqrt{n}T}$. In particular, there is a specific
assignment of values to $\sigma_t$ such that the regret of the
algorithm is at least $G\sqrt{\frac{1}{2}\tau\sqrt{n}T}$.
\end{proof}

The construction we used for deriving the above lower bound can be
easily adapted to derive a lower bound on the sample complexity of
learning the class $\mW$ in the batch setting. This is formalized in
\thmref{thm:collaborative-filtering-batch-LB}, which we restate here for
convenience.
\begin{itemize}
\item[] \textbf{\thmref{thm:collaborative-filtering-batch-LB} restated} The
  sample complexity of learning $\mW$ in the batch setting, is
  $\Omega(\tau \sqrt{n}/\epsilon^2)$. In particular, when $\tau = \Theta(n)$, the
  sample complexity is $\Omega(n^{1.5}/\epsilon^2)$.
\end{itemize}
\begin{proof}
  For simplicity, let us choose $m=n$. Let $k=\tau/\sqrt{n}$ and fix
  some small $\epsilon$.  Define
  a family of distributions over $[n]^2 \times \{-1, 1\}$ as
  follows. Each distribution is parameterized by a matrix $\bW$ such
  that there is some $I \subset [n]$, with $|I| = k$, where $W(i,j)
  \in \{-1, 1\}$ for $i \in I$ and $W(i,j)=0$ for $i \notin I$. Now,
  the probability to sample an example $(i,j,y)$ is
  $\left(\thalf+2\epsilon\right)\tfrac{1}{kn}$ if $i \in I$ and $y=W(i,j)$,
    is $\left(\thalf-2\epsilon\right)\tfrac{1}{kn}$ if $i \in I$ and
      $y=-W(i,j)$, and the probability is $0$ in all other cases.

  As in the proof of \thmref{thm:collaborative-filtering-LB}, any matrix
  defining such distribution is in $\mW$. Furthermore, if we consider
  the absolute loss function: $\ell(\bW,(i,j,y)) = \thalf
  |W(i,j)-y|$, then the expected loss of $\bW$ with respect to the
  distribution it defines is
\[
\E\left[  \thalf   |W(i,j)-y| \right] = \thalf-2\epsilon ~.
\]
In contrast, by standard no-free-lunch arguments, no algorithm can
know to predict an entry $(i,j)$ with error smaller than $\thalf -
\epsilon$ without observing $\Omega(1/\epsilon^2)$ examples from this
entry. Therefore, no algorithm can have an error smaller than
$\half-\epsilon$ without receiving $\Omega(kn/\epsilon^2)$ examples.
\end{proof}

\section{Implementation Details}
\label{sec:faster-alg}

In general, the update rule in Algorithm~\ref{mmw-alg} is a convex optimization problem and can be computed in polynomial time.
We now give the following more efficient implementation which takes essentially $\tilde{O}(p^3)$ time per round. This is based on the following theorem that is essentially proved in~\citet{TRW}:
\begin{Thm} \label{thm:dual}
  The optimal solution of $\arg\min_{\bX \in \K} \Delta(\bX, \bY)$, where $\bY$ is a given symmetric matrix, and 
  \[\K\ :=\ \{ \bX \in \sS^{n \times n}:\ \bA_j \bullet \bX \leq b_j \text{ for } j = 1, 2, \ldots, m\},\] is given by 
  \[ \bX^\star = \exp(\log(\bY) - \tsum_{j=1}^m \alpha_j^\star \bA'_j),\]
  where $\bA_j' = \frac{1}{2}(\bA_j + \bA_j^\top)$, and $\balpha^\star = \langle \alpha_1^\star, \alpha_2^\star, \ldots, \alpha_m^\star\rangle$ is given by
  \[ \balpha^\star\ =\ \arg\max_{\forall j \in [m]:\ \alpha_j \geq 0} -\Tr(\exp(\log(\bY) - \tsum_{j=1}^m \alpha_j \bA'_j)) - \tsum_{j=1}^m \alpha_j b_j.\]
\end{Thm}

The idea is to avoid taking projections on the set $\K$ in each round. If the chosen entry in round $t$ is $(i_t, j_t)$, then we compute $\bX_t$ as 
\[\bX_t\ =\ \arg\min_{\bX \in \K_t} \Delta(\bX, \exp(\log(\bX_{t-1} - \eta \bL_{t-1})),\] where the polytope $\K_t$ is defined as
\begin{align}
  \K_t\ :=\ \Biggl\{&~~\bX \in  \sS^{2p \times 2p}\ \text{ s.t. } \notag \\
& X(i_t, i_t) + X(j_t+q,j_t+q) + X(p+i_t,p+i_t) + X(p+j_t+q,p+j_t+q) \leq 4\beta \notag \\
& X(i_t, j_t+q) - X(p+i_t,p+j_t+q))\ \leq\ 1 \notag\\
& X(p+i_t,p+j_t+q)) - X(i_t, j_t+q)\ \leq\ 1 \notag\\
& \Tr(\bX)\ \leq\ \tau \notag
~~\Biggr\} \notag
\end{align}
The observation is that this suffices for the regret bound of Theorem~\ref{thm:oco-regret} to hold since the optimal point in hindsight $\bX^\star \in \K_t$ for all $t$ (see the proof of Theorem~\ref{thm:mmw-regret}). 

Note that $\K_t$ is defined using just $4$ constraints, and hence the dual problem given in Theorem~\ref{thm:dual} has only $4$ variables $\alpha_j$. Thus, standard convex optimization techniques (say, the ellipsoid method) can be used to solve the dual problem to $\epsilon$-precision in $O(\log(1/\epsilon))$ iterations, each of which requires computing the gradient and/or the Hessian of the objective, which can be done in $O(p^3)$ time via the eigendecomposition, leading to an $\tilde{O}(p^3)$ time algorithm overall.

More precisely, the iteration count for convex optimization methods have logarithmic dependence on the range of the $\alpha_j$ variables. Since $\Tr(\bX_{t-1}) \leq \tau$, we see (using the Golden-Thompson inequality~\citep{golden, thompson}) that 
\[\Tr(\exp(\log(\bX_{t-1} - \eta \bL_{t-1})))\ \leq\ \bX_{t-1} \bullet \exp(-\eta \bL_{t-1})\ \leq\ 3\tau.\]
Thus, setting all $\alpha_j = 0$, the dual objective value is at least $-3\tau$. Since $b_j \geq 1$ for all $j$, we get that the optimal values of $\alpha_j$ are all bounded by $3\tau$. Thus, the range of all $\alpha_j$ can be set to $[0, 3\tau]$, giving a $O(\log(\frac{\tau}{\epsilon}))$  bound on the number of iterations.


\section{Conclusions}

In recent years the FTRL (Follow The Regularized Leader) paradigm has become the method of choice for proving regret bounds for online learning problems. In several online learning problems a direct application of this paradigm has failed to give tight regret bounds due to suboptimal ``convexification'' of the problem. This unsatisfying situation occurred in mainstream applications, such as online collaborative filtering, but also in basic prediction settings such as the online max cut or online gambling settings.

In this paper we single out a common property of these unresolved problems: they involve {\it structured matrix} prediction, in the sense that the matrices involved have certain nice decompositions. We give a unified formulation for three of these structured matrix prediction problems which leads to near-optimal convexification. Applying the standard FTRL algorithm, Matrix Multiplicative Weights, now gives efficient and near optimal regret algorithms for these problems. In the process we resolve two COLT open problems. The main conclusion of this paper is that spectral analysis in matrix predictions tasks can be surprisingly powerful, even when the connection between the spectrum and the problem may not be obvious on first sight (such as in the online gambling problem). 

We leave open the question of  bridging the logarithmic gap between known upper and lower bounds for regret in these structured prediction problems. Note that since all the three decompositions in this paper are optimal up to constant factors, one cannot close the gap by improving the decomposition; some fundamentally different algorithm seems necessary. 
It would also be interesting to see more applications of the $(\beta, \tau)$-decomposition for other online matrix prediction problems.

\bibliographystyle{plainnat}
\bibliography{bib} 

\begin{thebibliography}{13}
\providecommand{\natexlab}[1]{#1}
\providecommand{\url}[1]{\texttt{#1}}
\expandafter\ifx\csname urlstyle\endcsname\relax
  \providecommand{\doi}[1]{doi: #1}\else
  \providecommand{\doi}{doi: \begingroup \urlstyle{rm}\Url}\fi

\bibitem[Abernethy(2010)]{Jake}
J.~Abernethy.
\newblock Can we learn to gamble efficiently?
\newblock In \emph{COLT}, 2010.
\newblock Open Problem.

\bibitem[Arora and Kale(2007)]{AroraKale}
S.~Arora and S.~Kale.
\newblock A combinatorial, primal-dual approach to semidefinite programs.
\newblock In \emph{STOC}, pages 227--236, 2007.

\bibitem[Cesa-Bianchi and Shamir(2011)]{CesaBianchiShamir2011}
N.~Cesa-Bianchi and O.~Shamir.
\newblock Efficient online learning via randomized rounding.
\newblock In \emph{25th Annual Conference on Neural Information Processing
  Systems (NIPS)}, 2011.

\bibitem[Elkies(2011)]{mathoverflow}
N.~D. Elkies.
\newblock 2-norm of the upper triangular ``all-ones'' matrix.
\newblock
  http://mathoverflow.net/questions/72361/2-norm-of-the-upper-triangular-all-ones-matrix,
  2011.

\bibitem[{Golden}(1965)]{golden}
S.~{Golden}.
\newblock {Lower Bounds for the Helmholtz Function}.
\newblock \emph{Physical Review}, 137:\penalty0 1127--1128, February 1965.
\newblock \doi{10.1103/PhysRev.137.B1127}.

\bibitem[Kakade et~al.(2010)Kakade, Shalev-Shwartz, and
  Tewari]{kakade2009duality}
S.~Kakade, S.~Shalev-Shwartz, and A.~Tewari.
\newblock Regularization techniques for learning with matrices.
\newblock \emph{preprint arXiv:0910.0610}, 2010.

\bibitem[Kanade and Steinke(2012)]{kanade}
V.~Kanade and T.~Steinke.
\newblock Learning hurdles for sleeping experts.
\newblock In \emph{Innovations in Theoretical Computer Science}, 2012.

\bibitem[Kleinberg et~al.(2010)Kleinberg, Niculescu-Mizil, and Sharma]{KNW}
R.~Kleinberg, A.~Niculescu-Mizil, and Y.~Sharma.
\newblock Regret bounds for sleeping experts and bandits.
\newblock \emph{Machine learning}, 80\penalty0 (2):\penalty0 245--272, 2010.

\bibitem[Lee et~al.(2010)Lee, Recht, Salakhutdinov, Srebro, and Tropp]{LRSST}
J.~Lee, B.~Recht, R.~Salakhutdinov, N.~Srebro, and J.~A. Tropp.
\newblock Practical large-scale optimization for max-norm regularization.
\newblock In \emph{NIPS}, pages 1297--1305, 2010.

\bibitem[Shamir and Shalev-Shwartz(2011)]{ShamirSh11}
O.~Shamir and S.~Shalev-Shwartz.
\newblock Collaborative filtering with the trace norm: Learning, bounding, and
  transducing.
\newblock In \emph{24th Annual Conference on Learning Theory (COLT)}, 2011.

\bibitem[Shamir and Srebro(2011)]{ShamirSr11}
O.~Shamir and N.~Srebro.
\newblock Sample complexity of trace-norm?
\newblock In \emph{COLT}, 2011.
\newblock Open Problem.

\bibitem[Thompson(1965)]{thompson}
C.~J. Thompson.
\newblock Inequality with applications in statistical mechanics.
\newblock \emph{Journal of Mathematical Physics}, 6\penalty0 (11):\penalty0
  1812--1823, 1965.

\bibitem[Tsuda et~al.(2006)Tsuda, Ratsch, and Warmuth]{TRW}
K.~Tsuda, G.~Ratsch, and M.K. Warmuth.
\newblock Matrix exponentiated gradient updates for on-line learning and
  bregman projection.
\newblock \emph{Journal of Machine Learning Research}, 6\penalty0 (1):\penalty0
  995, 2006.

\end{thebibliography}

\appendix

\section{Matrix Multiplicative Weights Algorithm}
\label{sec:mmw-analysis}

For the sake of completeness, we prove Theorem~\ref{thm:mmw-regret}. The setting is as follows. We have an online convex optimization problem where the decision set is a convex subset $\K$ of $N \times N$ positive semidefinite matrices of trace bounded by $\tau$, viz. for all $\bX \in \K$, we have $\bX \succeq \bzero$ and $\Tr(\bX) \leq \tau$. We assume for convenience that $\frac{\tau}{N}\eye \in \K$. In each round $t$, the learner produces a matrix $\bX_t \in \K$, and the adversary supplies a loss matrix $\bL_t \in \mathbb{R}^{N \times N}$, which is assumed to be symmetric. The loss of the learner is $\bX_t \bullet \bL_t$. The goal is to minimize regret defined as
\[ \text{Regret}\ :=\ \sum_{t=1}^T \bX_t \bullet \bL_t - \min_{\bX \in \K} \sum_{t=1}^T \bX \bullet \bL_t. \]
Consider Algorithm~\ref{mmw-alg}. We now prove Theorem~\ref{thm:mmw-regret}, which we restate here for convenience:
\begin{Thm}
	Suppose $\eta$ is chosen so that $\eta\|\bL_t\| \leq 1$ for all $t$. Then
	\[ \text{Regret}\ \leq\ \eta \sum_{t=1}^T \bX_t \bullet \bL_t^2 + \frac{\tau\log(N)}{\eta}. \]
\end{Thm}
\begin{proof}
	Consider any round $t$.
	Let $\bX \in \K$ be any matrix.	We use the quantum relative entropy, $\Delta(\bX, \bX_t)$, as a potential function. We have
	\begin{equation}
		\label{eq:potential-change}
		\Delta(\bX, \exp(\log(\bX_t) - \eta \bL_t)) - \Delta(\bX, \bX_t)\ =\ \eta \bX \bullet \bL_t - \Tr(\bX_t) + \Tr(\exp(\log(\bX_t) - \eta \bL_t)).
	\end{equation}
	Now quantum relative entropy projection onto the set $\K$ is a Bregman projection, and hence the Generalized Pythagorean inequality applies (see \citet{TRW}):
	\[ \Delta(\bX, \bX_{t+1}) + \Delta(\bX_{t+1}, \exp(\log(\bX_t) - \eta \bL_t)))\ \leq\ \Delta(\bX, \exp(\log(\bX_t) - \eta \bL_t))), \]
	and since $\Delta(\bX_{t+1}, \exp(\log(\bX_t) - \eta \bL_t))) \geq 0$, we get that
	\[ \Delta(\bX, \bX_{t+1})\ \leq\ \Delta(\bX, \exp(\log(\bX_t) - \eta \bL_t))).\]
	Hence from (\ref{eq:potential-change}) we get
	\begin{equation}
		\label{eq:potential-X_t+1}
		\Delta(\bX, \bX_{t+1}) - \Delta(\bX, \bX_t)\ \leq\ \eta \bX \bullet \bL_t - \Tr(\bX_t) + \Tr(\exp(\log(\bX_t) - \eta \bL_t)).
	\end{equation}
	Now, using the Golden-Thompson inequality~\citep{golden, thompson}, we have
	\begin{align*}
		\Tr(\exp(\log(\bX_t) - \eta \bL_t))\ &\leq\ \Tr(\bX_t
                \exp(-\eta \bL_t))
              \end{align*}
Next, using the fact that $\exp(\bA) \preceq \eye + \bA + \bA^2$ for
$\|\bA\| \leq 1$,\footnote{To see this, note that we can write $\bA =
  \bV \bD \bV^\top$ for some orthonormal $\bV$ and diagonal
  $\bD$. Therefore,
\[
\eye+\bA+\bA^2-e^{\bA} = \bV\left(  \eye + \bD + \bD^2 -
  e^{\bD}\right) V^\top ~.
\] Now, by the inequality $1+a+a^2-e^a \ge 0$, which holds for all $a
\le 1$, we obtain that all elements of the diagonal matrix $\left(  \eye + \bD + \bD^2 -
  e^{\bD}\right)$ are non-negative. 
}  we obtain
\begin{align*}
\Tr(\bX_t  \exp(-\eta \bL_t))&\leq\ \Tr(X_t(\eye - \eta \bL_t + \eta^2 \bL_t^2)\\
		&=\ \Tr(\bX_t) - \eta \bX_t \bullet \bL_t + \eta^2\bX_t \bullet \bL_t^2.
	\end{align*}
Combining the above and plugging into (\ref{eq:potential-X_t+1}) we get
	\begin{equation}
		\label{eq:potential-change-final}
		\Delta(\bX, \bX_{t+1}) - \Delta(\bX, \bX_t)\ \leq\ \eta \bX \bullet \bL_t - \eta \bX_t \bullet \bL_t + \eta^2\bX_t \bullet \bL_t^2.
	\end{equation}
	Summing up from $t = 1$ to $T$, and rearranging, we get
	\begin{align*}
		\text{Regret}\ &\leq\ \eta \sum_{t=1}^T \bX_t \bullet \bL_t^2 + \frac{\Delta(\bX, \bX_1) - \Delta(\bX, \bX_{T+1})}{\eta}\\
		&\leq\ \eta \sum_{t=1}^T \bX_t \bullet \bL_t^2 + \frac{\tau\log(N)}{\eta},
	\end{align*}
	since $\Delta(\bX, \bX_{T+1}) \geq 0$ and
  \begin{align*}
    \Delta(\bX, \bX_1)\ &=\ \bX \bullet (\log(\bX) - \log(\tfrac{\tau}{N}\eye)) - \Tr(\bX) + \tau\\
    &=\ \bX \bullet \log(\tfrac{1}{\tau}\bX) + \log(\tau)\Tr(\bX) - \log(\tfrac{\tau}{N})\Tr(\bX) - \Tr(\bX) + \tau\\
    &\leq\ \Tr(\bX)(\log(N) - 1) + \tau\\
    &\leq\ \tau \log(N).
  \end{align*}
  The first inequality above follows because $\Tr(\bX) \leq \tau$, so $\log(\tfrac{1}{\tau}\bX) \prec \bzero$. The second inequality uses $\Tr(\bX) \leq \tau$.

\end{proof}

\section{Technical Lemmas and Proofs}
\label{sec:technical}

\begin{Lem} \label{lem:sym}
	For $m \times n$ non-symmetric matrices $\bW$, if $\bW = \bU\bSigma \bV^\top$ is the singular value decomposition of $\bW$, then 
	\[\sym(\bW) = \left[ \begin{array}{cc}
\frac{1}{\sqrt{2}}\bU & \frac{1}{\sqrt{2}}\bU \\
\frac{1}{\sqrt{2}}\bV & -\frac{1}{\sqrt{2}}\bV  \end{array} \right]\left[ \begin{array}{cc}
\bSigma & \bzero \\
\bzero & -\bSigma  \end{array} \right] \left[ \begin{array}{cc}
\frac{1}{\sqrt{2}}\bU^\top & \frac{1}{\sqrt{2}}\bV^\top \\
\frac{1}{\sqrt{2}}\bU^\top & -\frac{1}{\sqrt{2}}\bV^\top  \end{array} \right]\]
is the eigenvalue decomposition of $\sym(\bW)$. In particular, 
$\|\sym(\bW)\|_\star = 2\|\bW\|_\star$.	
\end{Lem}
\begin{proof}
By the block matrix multiplication rule we have
\begin{align*}
& \left[ \begin{array}{cc}
\frac{1}{\sqrt{2}}\bU & \frac{1}{\sqrt{2}}\bU \\
\frac{1}{\sqrt{2}}\bV & -\frac{1}{\sqrt{2}}\bV  \end{array} \right]\left[ \begin{array}{cc}
\bSigma & \bzero \\
\bzero & -\bSigma  \end{array} \right] \left[ \begin{array}{cc}
\frac{1}{\sqrt{2}}\bU^\top & \frac{1}{\sqrt{2}}\bV^\top \\
\frac{1}{\sqrt{2}}\bU^\top & -\frac{1}{\sqrt{2}}\bV^\top  \end{array}
\right] \\
&= 
\left[ \begin{array}{cc}
\frac{1}{\sqrt{2}}\bU \bSigma & -\frac{1}{\sqrt{2}}\bU \bSigma\\
\frac{1}{\sqrt{2}}\bV \bSigma & \frac{1}{\sqrt{2}}\bV
\bSigma \end{array} \right]  \left[ \begin{array}{cc}
\frac{1}{\sqrt{2}}\bU^\top & \frac{1}{\sqrt{2}}\bV^\top \\
\frac{1}{\sqrt{2}}\bU^\top & -\frac{1}{\sqrt{2}}\bV^\top  \end{array}
\right] \\
&=\left[ \begin{array}{cc}
\bzero & \bU \bSigma \bV^\top \\
\bV \bSigma \bU^\top & \bzero\end{array} \right] \\
&= 
\left[ \begin{array}{cc}
\bzero & \bW \\
\bW^\top & \bzero  \end{array} \right].
\end{align*}
In addition, it is easy to check that the columns of $\left[ \begin{array}{cc}
\frac{1}{\sqrt{2}}\bU & \frac{1}{\sqrt{2}}\bU \\
\frac{1}{\sqrt{2}}\bV & -\frac{1}{\sqrt{2}}\bV  \end{array} \right]$ are orthonormal.
It follows that the above form is the eigendecomposition of
$\sym(\bW)$. Therefore, for any Schatten norm: $\|\sym(\bW)\| = 2\|\bSigma\| =
2 \|\bW\|$, which concludes our proof. 
\end{proof}

\section{The optimal cut in the Online Max Cut problem}
\label{sec:max-cut-proof}

We prove Lemma \ref{lem:maxcutoffline}, which we restate here for
convenience.
\begin{itemize}
\item[] \textbf{\lemref{lem:maxcutoffline} restated}  Consider an online sequence of loss functions $\{\ell_t = \frac{1}{2} |y_t - \hat{y}_y| \}$. Let
$$ \bW^* = \arg \min_{\bW \in \mW}  \sum_t \ell_t(W(i_t,j_t)) ~. $$
Then $\bW^* = \bW_A$ for the set $A$ that determines the max cut in
the weighted graph over $[n]$ nodes whose weights are given by $w_{ij}
= \sum_{t : (i_t,j_t) = (i,j)} y_t$ for every $(i,j)$.
\end{itemize}
\begin{proof}
Consider $\bW_A$. 
For each pair $(i, j)$ let $c_{ij}^+,c_{ij}^-$ be the total number of iterations in which the pair $(i,j)$ appeared in the adversarial sequence with $y_t = 1$ or $y_t = -1$ respectively. Since $\hat{y}_t \in [-1,1]$ we can rewrite the total loss as:
\begin{eqnarray*} 
\sum_t \ell_t(\bW_A(i_t,j_t)) & = & \frac{1}{2} \sum_{(i,j)} [ c_{ij}^+ \cdot (1 - W_A(i,j)) + c_{ij}^- \cdot (1 + W_A(i,j)) ] \\
& = & \frac{1}{2} \sum_{(i,j)} W_A(i,j) \cdot (c_{ij}^- - c_{ij}^+) + C_T \\
& = & - \frac{1}{2} \sum_{ (i,j)} W_A(i,j) \cdot w_{ij}  + C_T 
\end{eqnarray*}
Where $C_T$ is a constant which is independent of $\bW_A$. Hence, minimizing the above expression is equivalent to maximizing the expression:
\[\sum_{(i,j)} W_A(i,j) \cdot w_{ij}\ =\ \mathop{2\cdot\sum}_{(i, j):\ W_A(i, j)=1} w_{ij} - \sum_{(i, j)} w_{ij}.\]
Since $\sum_{(i, j)} w_{ij}$ is a constant independent of $A$, the cut which maximizes this expression is the maximum cut in the weighted graph over the weights $w_{ij}$.  
\end{proof}

\section{Optimality of Decomposition for Collaborative Filtering}
\label{sec:cf-optimality}

In this section, we prove the following theorem:
\begin{Thm} \label{thm:ocf-decomp-optimal}
  Consider the matrix $\bW$ formed by taking $m = \frac{\tau}{\sqrt{n}}$ rows of an $n \times n$ Hadamard matrix. This matrix has $\|\bW\|_\star = \tau$, and any $(\beta, \tilde{\tau})$-decomposition for $\sym(\bW)$ has 
  \[\beta \tilde{\tau}\ \geq\ \frac{1}{4}\tau \sqrt{n}.\]
\end{Thm}
\begin{proof}
  Since the rows of $\bW$ are orthogonal to each other, the $m$ singular values of $\bW$ all equal $\sqrt{n}$, and thus $\|\bW\|_\star = m\sqrt{n} = \tau$. Further, the SVD of $\bW$ is (here, $\eye_m$ is the $m \times m$ identity matrix):
  \[ \bW\ =\ \bI_m (\sqrt{n} \eye_m) (\tfrac{1}{\sqrt{n}} \bW). \] 
  Using Lemma~\ref{lem:sym} the eigendecomposition of $\sym(\bW)$ can be written as
  \[ \sym(\bW)\ =\ \bU (\sqrt{n}\eye_m) \bU^\top + \bV (-\sqrt{n}\eye_m) \bV^\top, \]
  where
  \[ \bU = [\tfrac{1}{\sqrt{2}}\eye_m, \tfrac{1}{\sqrt{2n}}\bW]^\top \text{ and } \bV = [\tfrac{1}{\sqrt{2}}\eye_m, -\tfrac{1}{\sqrt{2n}}\bW]^\top\]
  are $p \times m$ matrices with orthonormal columns.

  Let $\sym(\bW) = \bP - \bN$ be a $(\beta, \tilde{\tau})$-decomposition. Now consider the following matrices: first, define the $p \times p$ diagonal matrix
  \[ \bD\ :=\ 
\left[ \begin{array}{cc}
\frac{1}{\sqrt{2m}}\eye_m & \bzero \\
\bzero & \frac{\sqrt{mn}}{2\sqrt{2}\tilde{\tau}}\eye_n  \end{array} \right].\]
  Finally, define the $p \times p$ positive semidefinite matrix
  \[ \bY\ :=\ \bD \bU\bU^\top \bD. \]
  Since $\bU$ has orthonormal columns we have $\bU\bU^\top \preceq \eye_p$, and so
  \[ \bY\ \preceq\ \bD \bI_p \bD\ =\ \bD^2.\]
  Now, consider 
  \begin{align*}
    \bY \bullet \sym(\bW)\ &=\ \bY \bullet (\bP - \bN)\\
    &\leq\ \bY \bullet \bP & (\because \bY, \bN \succeq \bzero, \text{ so } \bY \bullet \bN \geq 0)\\
    &\leq\ \bD^2 \bullet \bP & (\because \bY \preceq \bD^2)\\
    &=\ \sum_{i=1}^m \frac{1}{2m} P(i, i) + \sum_{i=m+1}^p \frac{mn}{8\tilde{\tau}} P(i, i)\\
    &\leq\ \frac{1}{2}\beta + \frac{mn}{8\tilde{\tau}}\tau,
  \end{align*}
  since $P(i, i) \leq \beta$ for all $i$ and $\Tr(\bP) \leq \tau$. We also have
  \begin{align*}
    \bY \bullet \sym(\bW)\ &=\ \Tr(\bD \bU\bU^\top \bD \sym(\bW))\\
    &=\ \Tr(\bU\bU^\top \bD \sym(\bW)\bD)\\
    &=\ \frac{\sqrt{n}}{4\tilde{\tau}} \Tr(\bU\bU^\top\sym(\bW)) 
    &(\because \bD \sym(\bW) \bD = \sym(\tfrac{\sqrt{n}}{4\tilde{\tau}}\bW))\\
    &=\ \frac{\sqrt{n}}{4\tilde{\tau}} \Tr(\bU\bU^\top[\bU (\sqrt{n}\eye_m) \bU^\top + \bV (-\sqrt{n}\eye_m) \bV^\top])\\
    &=\ \frac{mn}{4\tilde{\tau}},
  \end{align*}
  since $\bU^\top\bV = \bzero$.
  Putting the above two inequalities together, we have
  \[ \frac{mn}{4\tilde{\tau}}\ \leq\ \frac{1}{2}\beta + \frac{mn}{8\tilde{\tau}},\]
  which implies that 
  \[\beta \tilde{\tau}\ \geq\ \frac{1}{4}mn\ =\ \frac{1}{4}\tau \sqrt{n}\] 
  as required.
\end{proof}

\section{Relation between $(\beta, \tau)$-decomposition, max-norm and trace-norm}
\label{sec:norm-connection}

In this section, we consider $m \times n$ non-symmetric matrix $\bW$. The max-norm of $\bW$ is defined to be (see \citet{LRSST}) the value of the following SDP:
\begin{align}
  \min\ &t \notag\\
  \left[ \begin{array}{cc}
\bY_1 & \bW\\
\bW^\top & \bY_2  \end{array} \right]\ &\succeq\ \bzero \notag\\
\forall i \in [m], j \in [n]:\ Y_1(i, i),\ Y_2(j, j)\ &\leq\ t. \label{eq:max-norm-sdp}
\end{align}

The least possible $\beta$ in any $(\beta, \tau)$-decomposition for $\bW$ is given by the following SDP:
\begin{align}
  \min\ &\beta \notag\\
  \left[ \begin{array}{cc}
\bzero & \bW\\
\bW^\top & \bzero  \end{array} \right]\ &=\ \bP - \bN \notag \\
\bP,\ \bN\ &\succeq\ \bzero \notag \\
\forall i \in [m + n]:\ P(i, i),\ N(i, i)\ &\leq\ \beta. \label{eq:beta-sdp}
\end{align}

\begin{Thm} \label{thm:max-norm-connection}
  The least possible $\beta$ in any $(\beta, \tau)$-decomposition exactly equals half the max-norm of $\bW$.
\end{Thm}
\begin{proof} Let $t^*$ and $\beta^*$ be the optima of SDPs (\ref{eq:max-norm-sdp}) and (\ref{eq:beta-sdp}) respectively. Let $\bY_1,\ \bY_2$ be the optimal solution to SDP (\ref{eq:max-norm-sdp}), so that for all 
$i \in [m],\ j \in [n]$ we have $Y_1(i, i),\ Y_2(j, j) \leq t^*$. Consider the matrices
  \[ \bP = \frac{1}{2}\left[ \begin{array}{cc}
\bY_1 & \bW\\
\bW^\top & \bY_2  \end{array} \right] \text{ and } \bN = \frac{1}{2}\left[ \begin{array}{cc}
\bY_1 & -\bW\\
-\bW^\top & \bY_2  \end{array} \right] .\]
Using the feasibility of $\bY_1,\ \bY_2$ and Lemma~\ref{lem:psd-transform}, we get that $\bP, \bN \succeq \bzero$. Thus this is a feasible solution to SDP (\ref{eq:beta-sdp}). Hence, we conclude that $t^* \geq 2\beta^*$.

Now let $\bP,\ \bN$ be the optimal solution to SDP (\ref{eq:max-norm-sdp}), so that for all $i \in [m + n]$ we have $P(i, i),\ N(i, i) \leq \beta^*$. Consider the blocks of $\bP$ and $\bN$ formed by the first $m$ indices and the last $n$ indices:
\[ \bP = \left[ \begin{array}{cc}
\bP^A & \bP^B\\
\bP^C & \bP^D  \end{array} \right] \text{ and } \bN = \left[ \begin{array}{cc}
\bN^A & \bN^B\\
\bN^C & \bN^D  \end{array} \right].\]
Since $\bN \succeq \bzero$, by Lemma~\ref{lem:psd-transform} the following matrix is positive semidefinite as well:
\[ \bN'\ :=\ \left[ \begin{array}{cc}
\bN^A & -\bN^B\\
-\bN^C & \bN^D  \end{array} \right]\ \succeq\ \bzero.\]
So $\bP + \bN' \succeq \bzero$, i.e.
\[\bP + \bN'\ =\ \left[ \begin{array}{cc}
\bP^A + \bN^A & \bW\\
\bW^\top & \bP^D + \bN^D  \end{array} \right]\ \succeq\ \bzero.\]
Thus, $\bY_1 = \bP^A + \bN^A$ and $\bY_2 = \bP^D + \bN^D$ is a feasible solution to SDP (\ref{eq:max-norm-sdp}). Now for all $i \in [m]$ we have $Y_1(i, i) \leq P^A(i, i) + N^A(i, i) \leq 2\beta^*$, and similarly for all $j \in [n]$ we have $Y_2(j, j) \leq 2\beta^*$. Thus, we conclude that $t^* \leq 2\beta^*$.
\end{proof}

\begin{Lem} \label{lem:psd-transform}
  Let $\bP$ be a positive semidefinite matrix of order $m + n$ and let
  \[\bP\ =\ \left[ \begin{array}{cc}
\bP^A & \bP^B\\
\bP^C & \bP^D  \end{array} \right].\]
be the block decomposition of $\bP$ formed by the first $m$ indices and the last $n$ indices. Then the following matrix is positive semidefinite: 
\[\bP'\ :=\ \left[ \begin{array}{cc}
\bP^A & -\bP^B\\
-\bP^C & \bP^D  \end{array} \right].\]
\end{Lem}
\begin{proof}
  Since $\bP \succeq \bzero$, there are vectors $\bv_i$, for all $i, j \in [m + n]$ such that $P(i, j) = \bv_i \cdot \bv_j$. Then consider the vectors 
  \[ \bw_i\ :=\ \begin{cases}
    \bv_i & \text{ if } i \in [m]\\
    -\bv_i & \text{ otherwise.}
  \end{cases}\]
  It is easy to check that for all $i, j \in [m + n]$ we have $P'(i, j) = \bw_i \cdot \bw_j$. Thus, we conclude that $\bP' \succeq \bzero$. 
\end{proof}

Finally, we show the connection between the trace-norm and the least possible $\tau$ in any $(\beta, \tau)$-decomposition:
\begin{Thm} \label{thm:trace-norm-connection}
  The least possible $\tau$ in any $(\beta, \tau)$-decomposition exactly equals twice the trace-norm of $\bW$.
\end{Thm}
\begin{proof}
  Let $\tau^*$ be the least possible value of $\tau$ in any $(\beta, \tau)$-decomposition, and let $\bP, \bN$ be positive semidefinite matrices such that $\sym(\bW)= \bP - \bN$  and $\Tr(\bP) + \Tr(\bN) = \tau^*$. Then by triangle inequality, we have
  \[ \|\sym(\bW)\|_\star\ \leq\ \|\bP\|_\star + \|\bN\|_\star.\]
  Since $\|\sym(\bW)\|_\star = 2\|\bW\|_\star$, $\|\bP\|_\star = \Tr(\bP)$, and $\|\bN\|_\star = \Tr(\bN)$, we conclude that $\tau^* \geq 2\|\bW\|_\star$. Now, let
\[\sym(\bW) = \sum_i \lambda_i \bv_i\bv_i^\top\]
be the eigenvalue decomposition of $\sym(\bW)$. Now consider the positive semidefinite matrices
\[\bP = \sum_{i:\ \lambda_i \geq 0} \lambda_i \bv_i\bv_i^\top \text{ and } \bN 
= \sum_{i:\ \lambda_i < 0} -\lambda_i \bv_i\bv_i^\top.\]  Clearly $\sym(\bW) = \bP - \bN$, and 
\[\Tr(\bP) + \Tr(\bN)\ =\ \sum_i |\lambda_i|\ =\ \|\sym(\bW)\|_\star\ =\ 2\|\bW\|_\star.\]
Hence, $\tau^* \leq 2\|\bW\|_\star$, completing the proof.
\end{proof}
\end{document}